\documentclass[letterpaper, 10 pt, conference]{ieeeconf}  

\IEEEoverridecommandlockouts                              

\usepackage{graphics} 
\usepackage{graphicx}
\usepackage{epsfig} 
\usepackage{times} 
\usepackage{amsmath} 
\usepackage{amssymb}  
\usepackage[english]{babel}
\usepackage{amsthm}

\usepackage{bm}
\usepackage[linesnumbered,ruled,vlined]{algorithm2e}
\usepackage[noend]{algpseudocode}
\usepackage[font={small}]{caption}
\usepackage{subcaption}
\usepackage{physics}
\usepackage{xcolor}
\usepackage{tikz}

\newtheorem{Definition}{Definition}
\newtheorem{Theorem}{Theorem}

\usetikzlibrary {arrows.meta}
\usetikzlibrary{fit,
                positioning}
\usepackage[belowskip=1ex]{subcaption}  

\usetikzlibrary{calc} 
\usetikzlibrary{positioning}
\usepackage{balance}
\usepackage{soul}
\usetikzlibrary{arrows}
\usepackage{cite}
\usepackage{url}
\usepackage{multirow}
\usepackage{booktabs} 
\usepackage[backref]{hyperref} 
\hypersetup{
colorlinks=true,
linkcolor=black,
citecolor=black
}

\usepackage{amssymb}
\usepackage{amsfonts}
\usepackage{amsmath}
\usepackage{amsthm}
\usepackage{bm}

\usepackage[normalem]{ulem}                                        
\usepackage{marginnote}
\usepackage[textwidth=10ex,colorinlistoftodos]{todonotes}

\colorlet{mh}{red}
\colorlet{fwu}{red}
\colorlet{ywu}{blue}
\colorlet{kchen}{blue}
\colorlet{lchen}{green}
\colorlet{zbing}{green}
\colorlet{shaddadin}{purple}
\colorlet{iperez}{cyan}
\colorlet{schneider}{magenta}

\usepackage{amssymb}
\usepackage{mathtools}
\usepackage{sansmath}

\mathchardef\mhyphen="2D   

\newcommand{\RNum}[1]{\uppercase\expandafter{\romannumeral #1\relax}}

\usepackage{setspace}
\title{\LARGE \bf

Elliptical K-Nearest Neighbors - Path Optimization via Coulomb's Law and Invalid Vertices in C-space Obstacles
}

\author{Liding Zhang$^{1}$, Zhenshan Bing$^{1}$, Yu Zhang$^{1}$, Kuanqi Cai$^{1}$, Lingyun Chen$^{1}$, Fan Wu$^{1}$, \\ Sami Haddadin$^{1}$,  Alois Knoll$^{1\dagger}$ 
\thanks{$^{1}$L. Zhang, Z. Bing, Y. Zhang, K. Cai, L. Chen,  F. Wu, S. Haddadin and A. Knoll are with the Department of Informatics, Technical University of Munich, Germany.
{\tt\small liding.zhang@tum.de}}%
\thanks{$^{\dagger}$The authors acknowledge the financial support by the Bavarian State Ministry for Economic Affairs, Regional Development and Energy (StMWi) for the Lighthouse Initiative KI.FABRIK (Phase 1: Infrastructure and the research and development program under grant no. DIK0249).}
\thanks{Corresponding author: Zhenshan Bing.}
}
\begin{document}

\maketitle
\thispagestyle{empty}
\pagestyle{empty}

\begin{abstract}
Path planning has long been an important and active research area in robotics. To address challenges in high-dimensional motion planning, this study introduces the Force Direction Informed Trees (FDIT*), a sampling-based planner designed to enhance speed and cost-effectiveness in pathfinding. FDIT* builds upon the state-of-the-art informed sampling planner, the Effort Informed Trees (EIT*), by capitalizing on often-overlooked information in invalid vertices. It incorporates principles of physical force, particularly Coulomb's law. This approach proposes the elliptical $k$-nearest neighbors search method, enabling fast convergence navigation and avoiding high solution cost or infeasible paths by exploring more problem-specific search-worthy areas. It demonstrates benefits in search efficiency and cost reduction, particularly in confined, high-dimensional environments. It can be viewed as an extension of nearest neighbors search techniques. Fusing invalid vertex data with physical dynamics facilitates force-direction-based search regions, resulting in an improved convergence rate to the optimum. FDIT* outperforms existing single-query, sampling-based planners on the tested problems in $\mathbb{R}^4$ to $\mathbb{R}^{16}$ and has been demonstrated on a real-world mobile manipulation task.\\
%
%
%
\end{abstract}

\section{Introduction}
In the realm of robot motion planning, especially in high-dimensional spaces, sampling-based algorithms play a pivotal role in path planning\cite{cai2023sampling}. These
algorithms are rooted in the mathematical foundation of \textit{random geometric graphs} (RGG)~\cite{penrose2003random}, probabilistic models representing randomly distributed networks in \textit{configuration space} (\textit{$\mathcal{C}$-space}).
In \textit{$\mathcal{C}$-space}, RGG position vertices randomly, extending into complex, higher-dimensional regions. Edges form between vertices within a specified distance, denoted as the graph's radius. Key features involve random vertex distribution in an $n$-dimensional space and edge creation based on proximity, typically measured using Euclidean metrics. The graph's connectivity hinges on the chosen radius, with a critical radius value determining connectivity or disconnectivity~\cite{penrose2003random}.

\begin{figure}[t!]
    \centering
    \begin{tikzpicture}

    \node[inner sep=0pt] (russell) at (-2.5,0)
    {\includegraphics[width=0.45\textwidth]{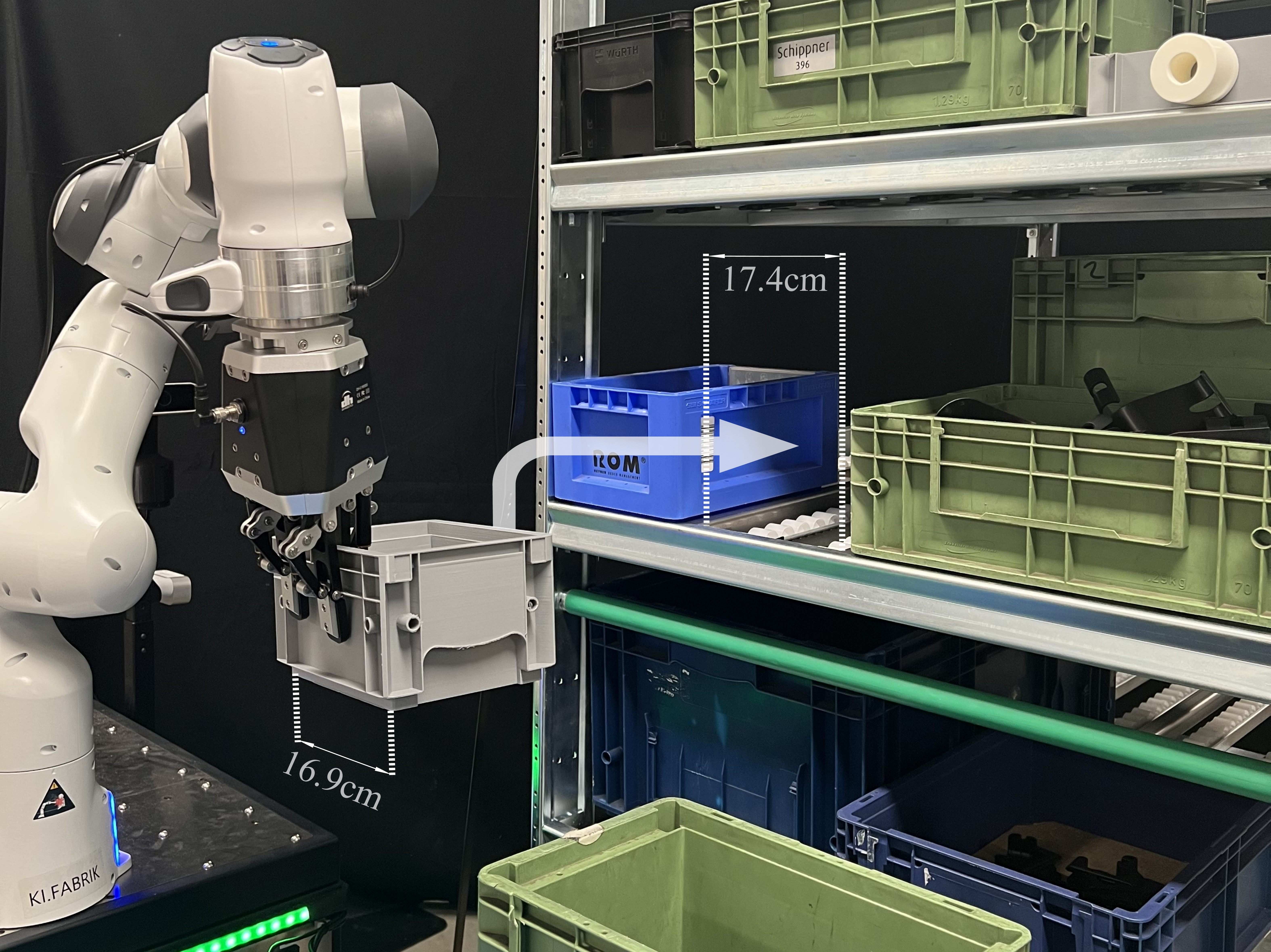}};

    \end{tikzpicture}
    \caption{Mobile manipulator robot during a real-time industry standardized ($\pm$5mm) container-shelf arrangement task (Section VI-B).}
    \label{fig: darko_setup}
    \vspace{-1.7em} 
\end{figure}

Motion planning, rooted in the A* algorithm~\cite{hart1968formal}, has developed in various directions. Lifelong planning A* (LPA*)\cite{koenig2004lifelong} addresses planning tasks in dynamic scenarios, while sampling-based planner Rapidly-exploring Random Trees (RRT)\cite{lavalle2001randomized} explores nonconvex and high-dimensional spaces. RRT builds a tree from a starting point, aiming for the goal point. RRT-Connect~\cite{connect2000} enhances efficiency by growing trees from both the start and goal positions. RRT*~\cite{karaman2011sampling} incrementally builds a tree using \textit{$k$-nearest}~\cite{fmt2015} search, providing anytime solutions and provably converging to an optimal solution. Informed RRT*~\cite{gammell2014informed} simplifies the search and sampling set using elliptical concepts (\textit{informed-sampling}~\cite{gammell2018informed}), thereby enhancing the convergence rate and the quality of the final solution during path optimization.

Batch Informed Trees (BIT*)\cite{gammell2020batch} integrates Informed RRT* and Fast Marching Trees (FMT*)~\cite{fmt2015}, combining single and multiple sampling techniques to compactly group states (\textit{batch}) into implicit RGG. Its advanced version, ABIT*~\cite{strub2020advanced}, uses an inflation factor and a truncation factor to balance exploration and exploitation in a dense RGG approximation.
%
Adaptively Informed Trees (AIT*)\cite{strub2022adaptively, strub2020adaptively} and Effort Informed Trees (EIT*)\cite{strub2022adaptively} employ an asymmetrical \textit{forward-reverse} search method with sparse collision checks in the reverse search. These adaptations enhance the efficiency of exploration in motion planning algorithms~\cite{zhang25g3t}.
It has evolved from evaluating single-dimensional points to higher-dimensional regions. 
However, the common utilization of finding nearest neighbors for path rewiring is either mutual \textit{$k$-nearest} or \textit{$r$-nearest}~\cite{Kleinbort_Salzman_Halperin_2020} neighbors, those methods neglect  information on obstacles and often abandon invalid vertices. The subsequent advancement involves exploring clusters of points or valuable regions, such as \textit{invalid} vertices.

\begin{figure}[t!]
    \centering
    \begin{tikzpicture}
    \node[inner sep=0pt] (russell) at (0.0, 0.0)
    {\includegraphics[width=0.48\textwidth]{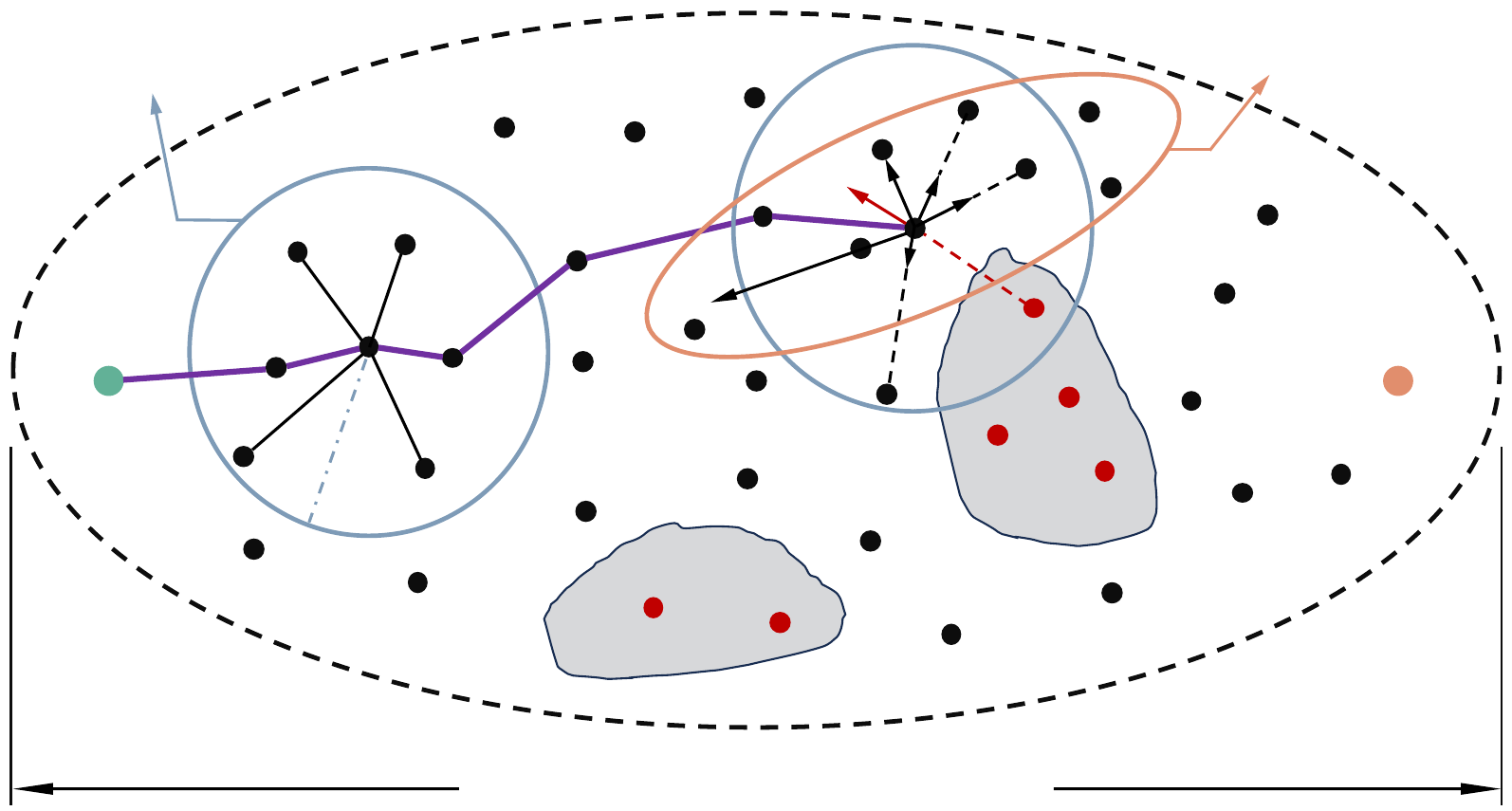}};
    \node at (-3.7,-0.12) {$\mathbf{x}_\text{start}$};
    \node at (3.7,-0.12) {$\mathbf{x}_\text{goal}$};
    \node at (-0.2,-0.9) {${X}_\text{obs}$};
    \node at (1.76,-0.53) {\color{purple}$V_\textit{invalid}$};
    \node at (-0.5,-0.1) {$V_\textit{valid}$};
    \node at (0.0,-2.22) {\textit{Informed Sampling}};
    \node at (3.4,1.98) {\footnotesize\textit{Elliptical-KNN}};
    \node at (-3.55,1.9) {\footnotesize\textit{KNN}};
    \node at (-2.25,-0.5) {\color{teal}$k$};
    
    \end{tikzpicture}
    \caption{The force directional elliptical $k$-nearest search is visualized. Black and red dots are sampled as valid and invalid vertices, $k$ donates as the connection parameters, black thin lines as forward rewire trees, magenta lines as the initial path, arrows indicate force directions and orange ellipses denote the area of elliptical-KNN in the forward search. Undesirable vertices are excluded from the neighbors, and relevant vertices are included in the elliptical-KNN.}
    \label{fig: illstra}
    \vspace{-1.7em} 
\end{figure}
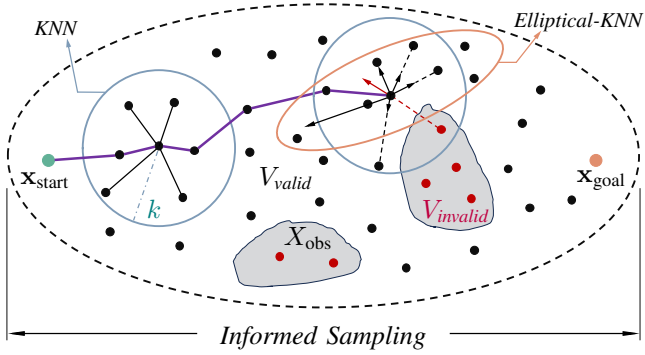

This paper introduces Force Direction Informed Trees (FDIT*), an asymptotically optimal sampling-based planner that employs a force-direction method and elliptical $k$-nearest neighbors (elliptical-KNN) search to enhance point evaluation techniques. These methods use invalid samples in $\mathcal{C}$-space obstacles to guide the search toward valuable areas, improving search efficiency in pathfinding.

FDIT* enhances exploration by using invalid vertices and physical forces. It extends traditional $k$-nearest neighbors search through Coulomb’s law-based force direction analysis, guiding the search and avoiding costly or infeasible areas. This improves search efficiency and reduces path costs, especially in high-dimensional, multi-narrow passage scenarios. Demonstrated in Fig.~\ref{fig: darko_setup}, FDIT* was tested on the DARKO robot for intralogistics tasks, showcasing efficient motion planning in confined environments. Its potential extends to automation and robotics, proving robust and effective in autonomous robotics and path planning.

The contributions of this work are summarized as follows:
\begin{itemize}
    \item \textit{Elliptical $k$-nearest neighbors:} FDIT* defines nearest neighbors by summing forces influenced by valid and invalid samples around current states. This enhances the likelihood of approaching more explore-worthy regions.
    \item \textit{Improved path quality:} FDIT* incorporates force direction analysis based on Coulomb’s law, resulting in more obstacle-awareness direction searches and up to 34.65\% initial solution cost reduction in various test scenarios.
    \item \textit{Real-world application in autonomous robotics:} FDIT* was tested in an industry-standard shelf-container arrangement task, and it attained the highest success rate (80\% under 1.0s) due to problem-specific exploration.
\end{itemize}

\section{Background\string& Preliminaries}
In an RGG, the connections (i.e., edges) between states (i.e., vertices or sample points) depend on their relative geometric positions. Common RGGs have edges to a specific number of each state's nearest neighbors (a $k$-nearest graph~\cite{Xue2004}) or all neighbors within a particular distance (an $r$-nearest graph~\cite{1961r_near}). 
RGG theory establishes probabilistic relationships between the number and distribution of samples, the $k$ or $r$ defining the graph, and specific graph properties such as connectivity or relative cost through the graph~\cite{karaman2011sampling}.
\subsection{$R$-Nearest Neighbors (RNN) Search}{
The connection radius $r(q)$ is a dynamic parameter that scales according to the density of the state space and the sample size $q$. It is defined as:
    \begin{equation}
    \label{eqn: radius KNN}
        r(q) := 2\eta \left(\left(1 + \frac{1}{n}\right){\left(\frac{\lambda(X_{\hat{f}})}{\lambda\left(B_{1, n}\right)}\right) \left( \frac{\log(q)}{q}\right)}\right)^{\frac{1}{n}},
    \end{equation}
where $\eta$ is a normalization constant, $n$ represents the dimensionality of the space, $\lambda(\cdot)$ denotes the Lebesgue measure, $X_{\hat{f}}$ is the informed set, and $B_{1,n}$ is a unit ball in an $n$-dimensional space~\cite{strub2022adaptively}.
This formula ensures that $r(q)$ adapts to the complexity of the space, enhancing the graph's connectivity without increasing computational effort~\cite{solovey2020critical}.
}
\subsection{$K$-Nearest Neighbors (KNN) Search}{
The number of nearest neighbors \( k(q) \), crucial for defining the graph's sparsity, is determined as follows:
\begin{equation}\label{eq: knn}
    k(q) := \eta e \left( 1 + \frac{1}{n} \right) \log(q),
\end{equation}

This formulation of \( k(q) \), as suggested in \cite{andoni2009nearest}, is designed to balance between thorough exploration of the state space and maintaining a manageable graph size. As the number of samples \( q \) increases, \( k(q) \) scales logarithmically, ensuring that each new node connects to a sufficient number of neighbors to maintain robustness in the graph structure.
}
\subsection{Coulomb's Law in Multi-Dimensional Space}
Coulomb's law, foundational in classical electrodynamics, explains the force between two point charges in three-dimensional space~\cite{lenard1961exact}.
The force is inversely proportional to the square of the distance between charges and directly proportional to the product of their magnitudes. Charges of the same type repel, while opposite charges attract. In an $n$-dimensional vacuum, the Coulomb's force magnitude between static charges is inversely proportional to the $n-1$ power of the distance.
Given two charges $p$ and $p^{\prime}$ with magnitudes $q$ and $q^{\prime}$, and positions $r$ and $r^{\prime}$, the Coulomb force vector from $q^{\prime}$ to $q$ is represented as:
\begin{equation}
    \vec{F}\left(p, p^{\prime}\right) = \mathcal{E} q^{\prime} \vec{r} = k_{n} \frac{q q^{\prime}}{\left\|r-r^{\prime}\right\|^{n-1}} \vec{r},
\end{equation}
where $\mathcal{E}$ denotes the electric field intensity at $p^{\prime}$, $k_{n}$ is the electrostatic constant in $n$-dimensional space, and $\vec{r}$ is the unit vector from $p^{\prime}$ to $p$.
The Coulomb force vector can be detailed in components as $\left[f_{1}\left(p, p^{\prime}\right), \ldots, f_{n}\left(p, p^{\prime}\right)\right]$, where $n$ is the dimensionality of the field. The Coulomb forces are unique and independent of dimension for any two distinct charges.
In an $n$-dimensional electrostatic field $Q$ composed of charges, the Coulomb resultant force on a charge $p$ is the sum of forces exerted by other charges, defined as:
\begin{equation}
    \overrightarrow{F_{Q}}(p) := \sum_{p^{\prime} \in Q \backslash p} \vec{F}\left(p, p^{\prime}\right) = \left[f_{1}(p), \ldots, f_{n}(p)\right].
\end{equation}
where $f_{n}(p)$ is the summation of the $n$-dimensional components of the Coulomb force on $p$ by other charges.
Coulomb forces and resultant forces provide the relative positioning and deviation of charge $p$ within the electrostatic field.

For sampling-based planners, the Open Motion Planning Library (OMPL) \cite{sucan2012open}, widely employed in benchmarking motion-planning algorithms, provides a comprehensive framework and tools for researchers to evaluate algorithms. Force Direction Informed Trees (FDIT*) is integrated into the OMPL framework, the Planner-Arena benchmark database \cite{moll2015benchmarking}, and Planner Developer Tools (PDT) \cite{gammell2022planner}
\section{Problem Formulation}
\begin{algorithm}[t!]
\caption{{get\_best\_ellipse\_k\_nearest\_samples}}
\label{alg:elliptcal_search}
\small
\DontPrintSemicolon
\SetKwInOut{Input}{Input}
\SetKwInOut{Output}{Output}
\SetKwFunction{ComputeForceDirection}{compute\_force\_direction}
\SetKwFunction{IsValid}{is\_valid}
\SetKwFunction{EllipseNeighbors}{ellipse\_neighbors}
\SetKwFunction{CalculateRatio}{calculate\_ratio}

\Input{$\mathbf{x}$ - The current state, $k$ - Number of nearest neighbors to find}
\Output{Best set of $k$ nearest neighbors within an ellipsoidal region around $\mathbf{x}$}

\BlankLine

${F}_\mathcal{\hat{D}} \gets \text{an $n$-dimensional vector of all ones}$\\ \Comment{initial force direction vector}\\
\( \textit{totalVertices} \gets 0 \) , \( \textit{invalidVertices} \gets 0 \)\;
$V_\textit{ellipseNeighbors} \gets \EllipseNeighbors(\mathbf{x}, k, {F}_\mathcal{\hat{D}})$\\
$\Phi \gets \CalculateRatio(V_\textit{ellipseNeighbors})$\\ \Comment{initialize charge ratio, skip loop when 90\% NN is valid}

\While(\Comment{NN invalid samples less than 10\%}){$\Phi \geq 0.1$}{
    ${F}_\mathcal{\hat{D}} \gets \ComputeForceDirection(\mathbf{x}, V_{\textit{ellipseNeighbors}})$\\
    \emph{\color{purple}$V_{\textit{ellipseNeighborsNew}} \gets \EllipseNeighbors(\mathbf{x}, k, {F}_\mathcal{\hat{D}})$}\\ \Comment{recalculate NN based on the updated force direction}

    \ForEach{\(\mathbf{x}_i \in V_\textit{ellipseNeighborsNew}\)}{
        \( \textit{totalVertices} \gets \textit{totalVertices} + 1 \)\;
        \If{\textbf{not} $\IsValid(\mathbf{x}_i)$}{
            \( \textit{invalidVertices} \gets \textit{invalidVertices} + 1 \)\;
        }
    }
    
    \eIf{\(\textit{totalVertices} > 0\)}{
        \( \Phi \gets {{\textit{invalidVertices}}}/{\textit{totalVertices}} \)\\
        \Comment{non-negative charge ratio $\in$ $(0,1]$}
    }{
        \( \Phi \gets 0 \)\Comment{no samples in nearest neighbor}
    }
    
    $V_\textit{{ellipseNeighbors}} \gets V_\textit{{ellipseNeighborsNew}}$
}
\ForEach{$\mathbf{x}_i \in V_\textit{{ellipseNeighbors}}$}{
    \If{\IsValid{$\mathbf{x}_i$}}{
        $V_\textit{{valid}} \gets V_\textit{{valid}} \cup \{\mathbf{x}_i\}$
    }
}
\Return $V_\textit{{ellipseNeighbors}}$
\end{algorithm}
\begin{algorithm}[t]
\caption{{ellipse\_neighbors($\mathbf{x}$, ${F}_\mathcal{\hat{D}}$})}
\label{alg:ellipse_neighbors}
\SetAlgoLined
\SetKwInOut{Input}{input}
\SetKwInOut{Output}{output}
\SetKwFunction{FEllipseNearest}{ellipse\_nearest}
\SetKwFunction{FParent}{parent}
\SetKwFunction{FChildren}{children}
\DontPrintSemicolon
\small
\Input{$\mathbf{x}$ - The current state, ${F}_\mathcal{\hat{D}}$ - Force direction vector}
\Output{$V_\textit{ellipseNeighbors}$ - The nearest neighbor vertices within the ellipsoidal region}

$V_\textit{ellipseNeighbors} \gets \FEllipseNearest(\mathbf{x}, k, {F}_\mathcal{\hat{D}})$\\
$V_\textit{ellipseNeighbors} \stackrel{+}{\leftarrow} V_\textit{ellipseNeighbors} \cup \{\FParent_\mathcal{F}(\mathbf{x}) \cup \FChildren_\mathcal{F}(\mathbf{x})\} \setminus V_\textit{ellipseNeighbors}$\\
$V_\textit{ellipseNeighbors} \stackrel{-}{\leftarrow} \{X_i \in V_\textit{ellipseNeighbors} | (X, X_i) \in E_{\mathcal{F},\textit{invalid}}\}$\\
\KwRet $V_\textit{ellipseNeighbors}$
\end{algorithm}
We define the optimal planning problem as outlined in~\cite{karaman2011sampling}. In this context:

\textit{Problem Definition 1 (Optimal Planning):} Given a planning problem with state space $X \subseteq \mathbb{R}^n$, where $X_{\text{obs}}$ denotes states in collision with obstacles, and $X_{\text{free}} = cl(X \setminus X_{\text{obs}})$ represents permissible states, where $cl(\cdot)$ represents the \textit{closure} of a set. The initial state is $\mathbf{x}_{\text{start}} \in X_{\text{free}}$, and the desired final states are in $X_{\text{goal}} \subset X_{\text{free}}$. A continuous map $\sigma: [0, 1] \mapsto X$ represents a collision-free path, and $\Sigma$ is the set of all nontrivial paths.

The optimal solution, denoted as $\sigma^*$, is the path minimizing a chosen cost function $s: \Sigma \mapsto \mathbb{R}{\geq 0}$. This path connects the initial state $\mathbf{x}_{\text{start}}$ to any goal state $\mathbf{x}_{\text{goal}} \in X_{\text{goal}}$ through the free space:
\begin{equation}
\begin{split}
    \sigma^* &= \arg \min_{\sigma \in \Sigma} \left\{ s(\sigma) \middle| \sigma(0) = \mathbf{x}_{\text{start}}, \sigma(1) \in \mathbf{x}_{\text{goal}}, \right. \\
    &\qquad\qquad \left. \forall t \in [0, 1], \sigma(t) \in X_{\text{free}} \right\}.
\end{split}
\end{equation}
where $\mathbb{R}_{\geq 0}$ denotes non-negative real numbers. The cost of the optimal path is $s^*$.

Considering a discrete set of states, $X_{\text{samples}} \subset X$, represented as a graph with edges determined algorithmically by a transition function, we can characterize its properties using a probabilistic model, specifically implicit dense RGGs, i.e., $X_{\text{samples}} = {\mathbf{x} \sim \mathcal{U}(X)}$, as discussed in~\cite{penrose2003random}.

\textit{Problem Definition 2 (Nearest Neighbors):}
A fundamental primitive that frequently arises is the \textit{nearest neighbors} (NN) problem. The problem is defined as follows: given a collection of $n$ objects, build a data structure that, upon receiving an arbitrary query object, reports the dataset object that is most similar to the query.
The concept of Randomized $c$-approximate RNNs, denoted as $(c, R)$-NN, is central in computational geometry. Consider a set $D$ comprising points in a $n$-dimensional space $\mathbb{R}^n$, along with parameters $R > 0$ and $\delta > 0$. The objective is to devise a data structure capable of, when presented with any query point $q$, identifying some $cR$-near neighbors of $q$ within $D$ with a probability of success no less than $1 - \delta$.
Assuming $\delta$ is a constant not close to $1$ (e.g., $1/2$), success probability can be boosted by building and querying multiple data structure instances~\cite{andoni2009nearest}.

\begin{figure*}[t!]
    \centering
    \begin{subfigure}[b]{0.32\textwidth}
        \includegraphics[width=\textwidth]{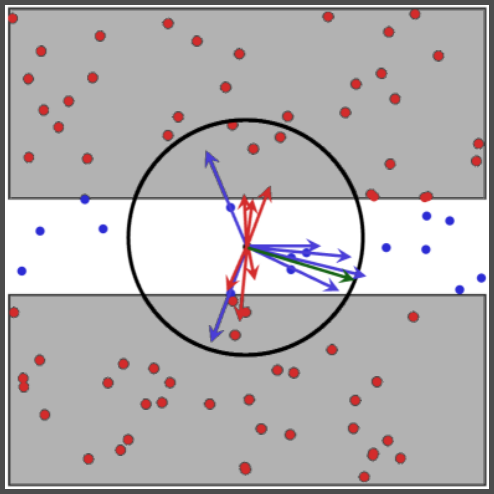}
        \caption{Initial iteration with a circular search area, balancing negative and positive forces, with \( k = 12 \), negative \( = 6 \), positive \( = 6 \).}
        \label{fig:circle1}
    \end{subfigure}
    \hfill
    \begin{subfigure}[b]{0.32\textwidth}
        \includegraphics[width=\textwidth]{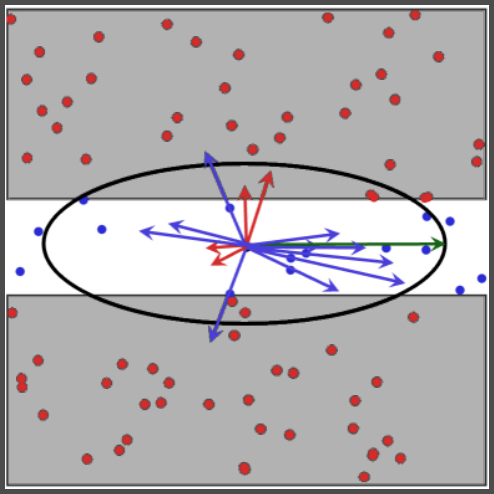}
        \caption{Subsequent iteration with an elliptical search area, influenced by the dominant resultant force, with \( k = 14 \), negative \( = 4 \), positive \( = 10 \).}
        \label{fig:circle2}
    \end{subfigure}
    \hfill
    \begin{subfigure}[b]{0.32\textwidth}
        \includegraphics[width=\textwidth]{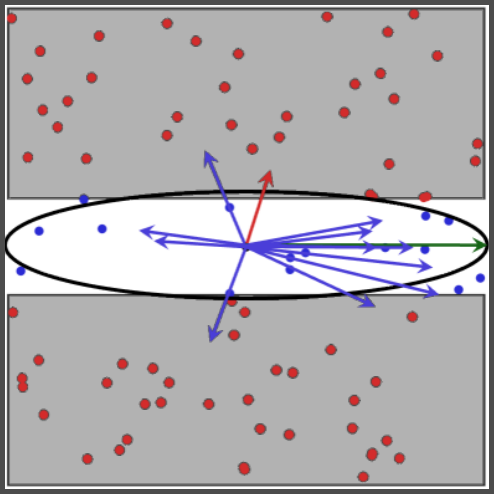}
        \caption{Final iteration with a refined search area, following the optimized direction of the resultant force, with \( k = 12 \), negative \( = 1 \), positive \( = 11 \).}
        \label{fig:circle3}
    \end{subfigure}
   
    \begin{tikzpicture}
        \node[inner sep=0pt] (russell) at (0.0,0.0)
        {\includegraphics[width=0.77\textwidth]{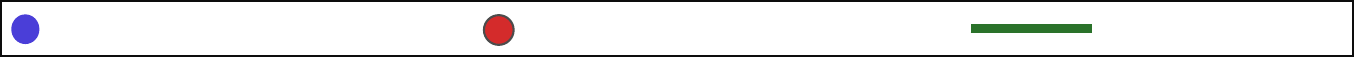}};
        \node at (-4.7,-0.03) {\small\text{Valid Samples (Positive)}};
        \node at (0.3,-0.03) {\small\text{Invalid Samples (Negative)}};
        \node at (5.45,0.0) {\small\text{Resultant Force}};
    \end{tikzpicture}
    \caption{Illustration of the progression of the search area refinement in FDIT* on a narrow corridor benchmark scenario, it demonstrates how the algorithm incrementally focuses the search towards the most non-obstacle feasible regions, as the cumulative resultant force directs (Alg.~\ref{alg:elliptcal_search}). Red points represent invalid samples that fall within obstacles. Blue points signify valid samples within obstacle-free areas. Corresponding arrows indicate the force exerted by each point on the current position. The green arrow is the resultant force that guides the search towards feasible regions. The gray areas represent the obstacle block's walls, and the white spaces delineate the accessible areas for path planning. \( k \) represents the total number of nearest neighbors considered during the search.}
    \label{fig:search_areas}
    \vspace{-1.7em} 
\end{figure*}

\section{Force Direction Informed Trees (FDIT*)}
\begin{algorithm}[t!]
\caption{{ellipse\_nearest($\mathbf{x}$, $X_\textit{all}$, ${F}_\mathcal{\hat{D}}$, $k$)}}
\label{alg:ellipseNearest}
\DontPrintSemicolon
\small
\SetKwInOut{Input}{Input}
\SetKwInOut{Output}{Output}
\SetKwFunction{FIsWithinEllipse}{is\_within\_ellipse}
\SetKwFunction{FIsValid}{is\_valid}
\SetKwFunction{Fsort}{sort}

\Input{$\mathbf{x}$ - The current state, $X_\textit{all}$ - All available vertices, ${F}_\mathcal{\hat{D}}$ - Force direction vector, $k$ - The number of nearest valid vertices to find}
\Output{$V_\textit{ellipseNearestValid}$ - The nearest $k$ valid vertices within the ellipsoidal region, $V_\textit{valid}$ - All valid vertices, $V_\textit{invalid}$ - All invalid vertices}
\textit{Initialization-filter vertices within the ellipsoidal region and separate valid and invalid vertices}\\
$V_\textit{sortedValid} \gets \emptyset$, $V_\textit{valid} \gets \emptyset$, $V_\textit{invalid} \gets \emptyset$, $V_\textit{ellipseNearest} \gets \emptyset$\;
\ForEach{$\mathbf{x}_i \in X_{all}$}{
  \If{\FIsWithinEllipse{$\mathbf{x}$, $\mathbf{x}_i$, ${F}_\mathcal{\hat{D}}$}}{
    \If{\FIsValid{$\mathbf{x}_i$}}{
      $V_\textit{valid} \gets V_\textit{valid} \cup \{\mathbf{x}_i\}$\;
    }{
      $V_\textit{invalid} \gets V_\textit{invalid} \cup \{\mathbf{x}_i\}$\;
    }
  }
}
$V_\textit{sortedValid} \gets$ \Fsort{$V_\textit{valid}$}\;
$V_\textit{ellipseNearestValid} \gets \text{first $k$ elements of } V_\textit{sortedValid}$\;

\Return{$V_\textit{ellipseNearestValid}$}\;
\end{algorithm}


FDIT* uses elliptical $k$-nearest neighbors and Coulomb's law-based force direction to guide the search process, steering it away from collisions and mimicking natural navigation patterns. A force computation formula (Alg.~\ref{computationforcedirection}) and elliptical-KNN search are introduced to explore neighboring areas.

\subsection{Notation}~\label{subsec: notation}
The state space of the planning problem is denoted by $X \subseteq \mathbb{R}^n$, where $n \in \mathbb{N}$. The start vertex is represented by $\mathbf{x}_{\text{start}} \in X$, and the goals are denoted by $X_{\text{goal}} \subset X$. The sampled states are denoted by $X_{\text{sampled}}$. 
The forward and reverse search trees are represented by $\mathcal{F}$ and $\mathcal{R}$, respectively. The vertices and edges in these trees, denoted by $V_\mathcal{F}$ and $V_\mathcal{R}$, are associated with valid states. 
The edges in the forward tree, $E_\mathcal{F} \subseteq V_\mathcal{F} \times V_\mathcal{F}$, represent valid connections between states.
An edge comprises a source state, $\mathbf{x}_s$, and a target state, $\mathbf{x}_t$, denoted as $(\mathbf{x}_s, \mathbf{x}_t)$. 
An admissible estimate for the cost of a path is $\hat{f}: X \rightarrow [0, \infty)$.
%
For sets $A, B,$ and $C$ with $B, C$ being subsets of $A$, the notation $B \stackrel{+}{\leftarrow} C$ denotes $B \leftarrow B \cup C$, and $B \stackrel{-}{\leftarrow} C$ denotes $B \leftarrow B \setminus C$.

\textit{FDIT*-specific Notation:}
Force calculations are conducted through a formula incorporating elliptical search neighbors, denoted as $V_{\textit{ellipseNeighbors}}$. Distinct electrical charges characterize samples within or outside obstacles: invalid charges $V_{\textit{invalid}}$ associated with energy $\mathcal{E}_{\textit{invalid}}$, and valid charges $V_{\textit{valid}}$ with energy $\mathcal{E}_{\textit{valid}}$. The charge ratio within $V_{\textit{ellipseNeighbors}}$ is denoted as $\Phi$, and the direction of the resultant force on each vertex is represented by $F_{\mathcal{\hat{D}}}$. The electric field of nearest neighbor samples within $V_{\textit{ellipseNeighbors}}$ is denoted as $Q_{\textit{valid}}$, $Q_{\textit{invalid}}$, respectively. Subsequent force calculations are executed on the acquired valid set to evaluate its quality. 

${F}_{\mathcal{\hat{D}}}$ serving as a vector to store the force direction information and operates within the same \textit{$\mathcal{C}$-space} dimension as the RGG. Spatial coordinates of valid and invalid sample points are respectively expressed as $\mathbf{x}_{\textit{positive}}$ and $\mathbf{x}_{\textit{negative}}$. The parameter $\rho_0$ remains constant in this context. 
The number of valid samples is denoted as $m_{\textit{valid}}$, while $m_{\textit{invalid}}$ represents the count of invalid sample points. Electric charges of samples, denoted by $q_1$ and $q_2$, correspond to the current vertex and a vertex from the surrounding set, both scalars. The attractive force exerted by valid sample points on the current vertex is represented by $\vec{F}_{\textit{attractive}}$, while $\vec{F}_{\textit{repulsive}}$ indicates the repulsive force exerted by invalid sample points. 
Magnetic fields vector generated by surrounding sample points and the energy of the current vertex within these magnetic fields are represented by $\vec{\mathcal{E}}_{\textit{attractive}}$ and $\vec{\mathcal{E}}_{\textit{repulsive}}$, respectively. 

The Euclidean distance between two samples is denoted as $r_i$, and the unit direction vector between two samples is represented by $\mathbf{\hat{r_i}}$. The magnitude and direction of the force exerted by the $i^{th}$ vertex on the current vertex are denoted by $\vec{F_i}$. Sets such as $V_{\textit{ellipseNeighborsNew}}$, $V_{\textit{ellipseNeighborsValid}}$, and $V_{\textit{sortedValid}}$ are utilized to store related vertices correspondingly. The non-negative scalar charge ratio $\Phi$ represents the scale between valid and invalid vertices.
\begin{algorithm}[t]
\caption{\small{ compute\_force\_direction($\mathbf{x}$, $V_\textit{{ellipseNeighbors}}$)}}
\label{computationforcedirection}
\SetKwInOut{Input}{input}
\SetKwInOut{Output}{output}
\SetKwFunction{FDistance}{distance}
\SetKwFunction{FIsValid}{is\_valid}
\SetKwFunction{FSingleForce}{single\_force} 
\small
\DontPrintSemicolon

\Input{$\mathbf{x}$ - The current state, $V_\textit{{ellipseNeighbors}}$ - A set of neighboring samples from ellipse neighbors}
\Output{${F}_\mathcal{\hat{D}}$ - The resultant force direction vector}

${F}_\mathcal{\hat{D}} \leftarrow \vec{0}$\Comment{initialize as a zero vector for safety} \\
\ForEach{$\mathbf{x}_i \in V_\textit{{ellipseNeighbors}}$}{
     $r_i \gets \FDistance(\mathbf{x}, \mathbf{x}_i)$ \Comment{compute the states distance}\\
      $\mathbf{\hat{r}}_i \gets \frac{\mathbf{x}_i - \mathbf{x}}{r_i}$  \Comment{compute the unit vector direction of $\mathbf{x}_i$}\\
      
      \eIf{\IsValid{$\mathbf{x}_i$}}{
       
        \emph{\color{purple}$\vec{F}_{\textit{attractive},i} \stackrel{+}{\leftarrow}  \frac{Q_\textit{valid}}{r_i^{n-1}} \mathbf{\hat{r}}_i$ \label{line:forceAttra}}\Comment{attractive positive force}\\
      }{
       
        \emph{\color{purple}$\vec{F}_{\textit{repulsive},i} \stackrel{+}{\leftarrow} - \frac{Q_\textit{invalid}}{r_i^{n-1}} \mathbf{\hat{r}}_i$} \Comment{repulsive negative force}\\
      }
    ${F}_\mathcal{\hat{D}} \leftarrow {F}_\mathcal{\hat{D}} \pm \vec{F}_i$ \Comment{add or minus depends on direction} 
}
\Return ${F}_\mathcal{\hat{D}}$\;
\end{algorithm}
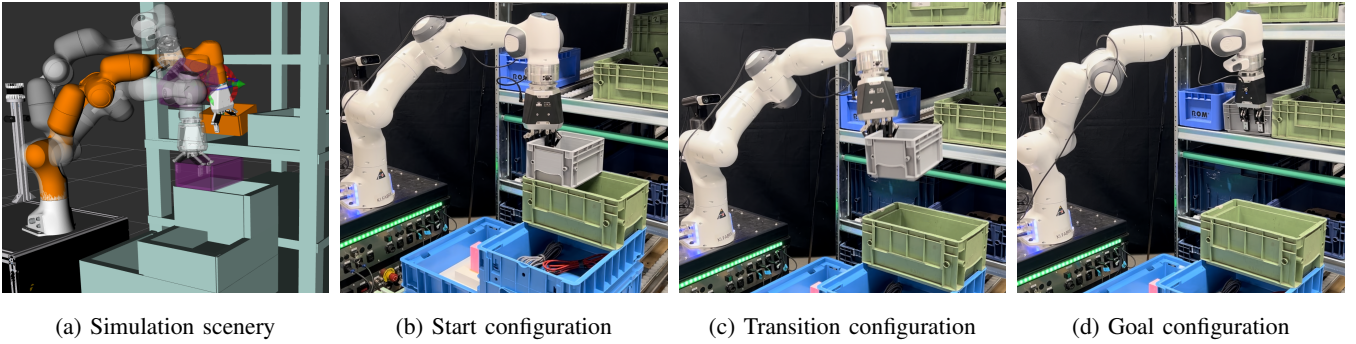
\begin{figure*}[t!]
    \centering
    \begin{tikzpicture}
    \footnotesize
    \node[inner sep=0pt] (russell) at (-8.0,0)
    {\includegraphics[width=0.244\textwidth]{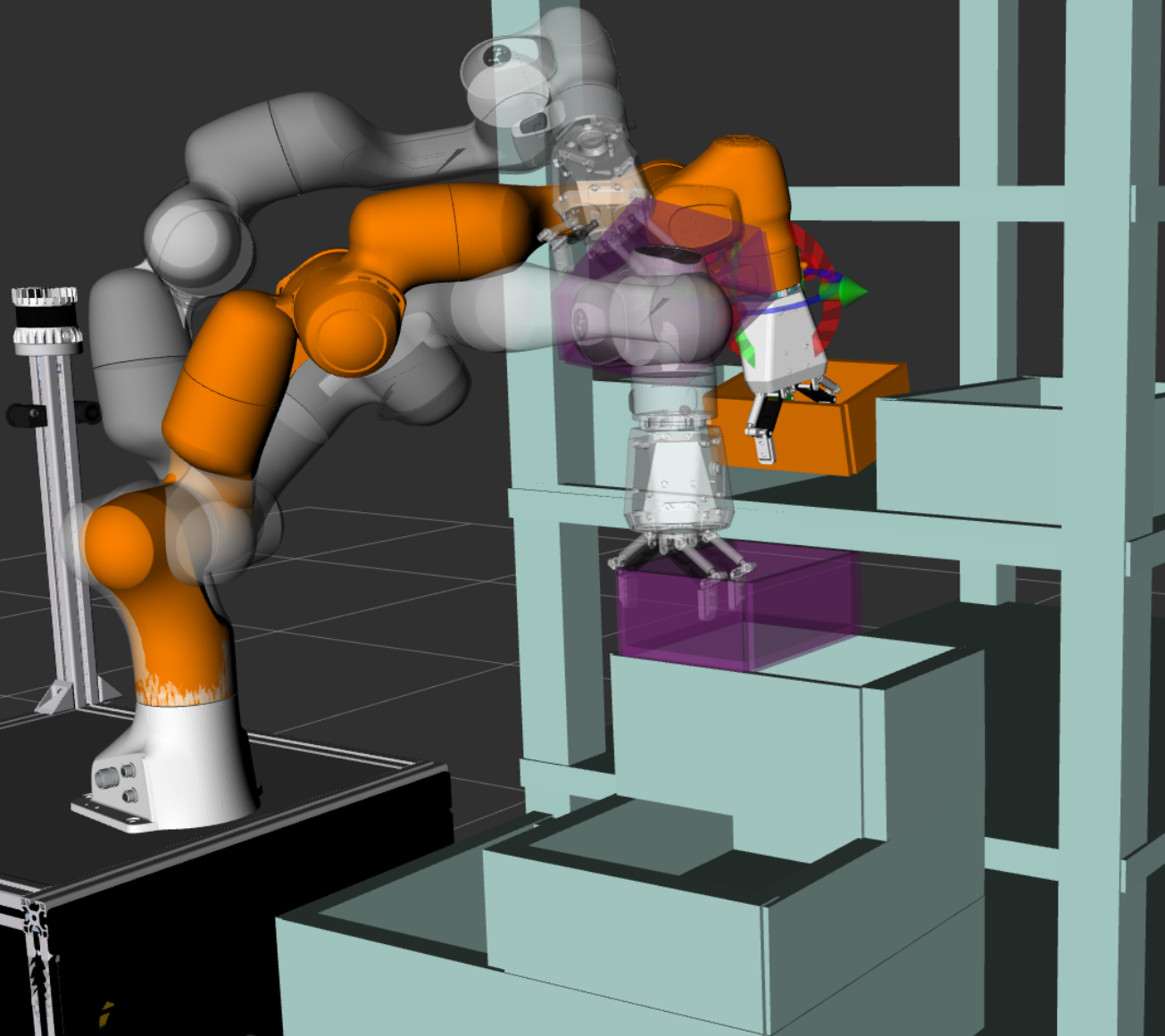}};
    \node[inner sep=0pt] (russell) at (-3.5,0)
    {\includegraphics[width=0.244\textwidth]{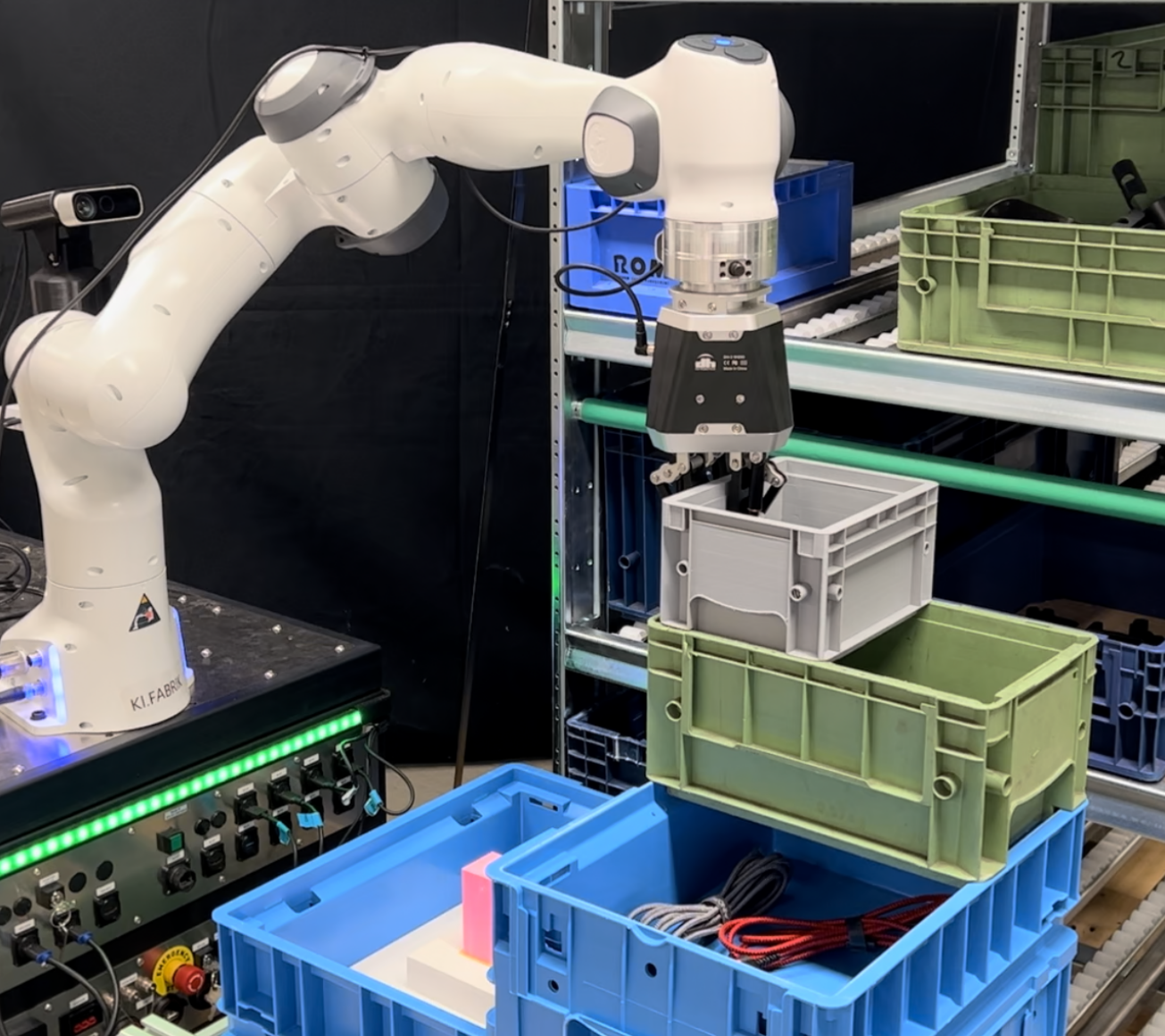}};
    \node[inner sep=0pt] (russell) at (1.0,0)
    {\includegraphics[width=0.244\textwidth]{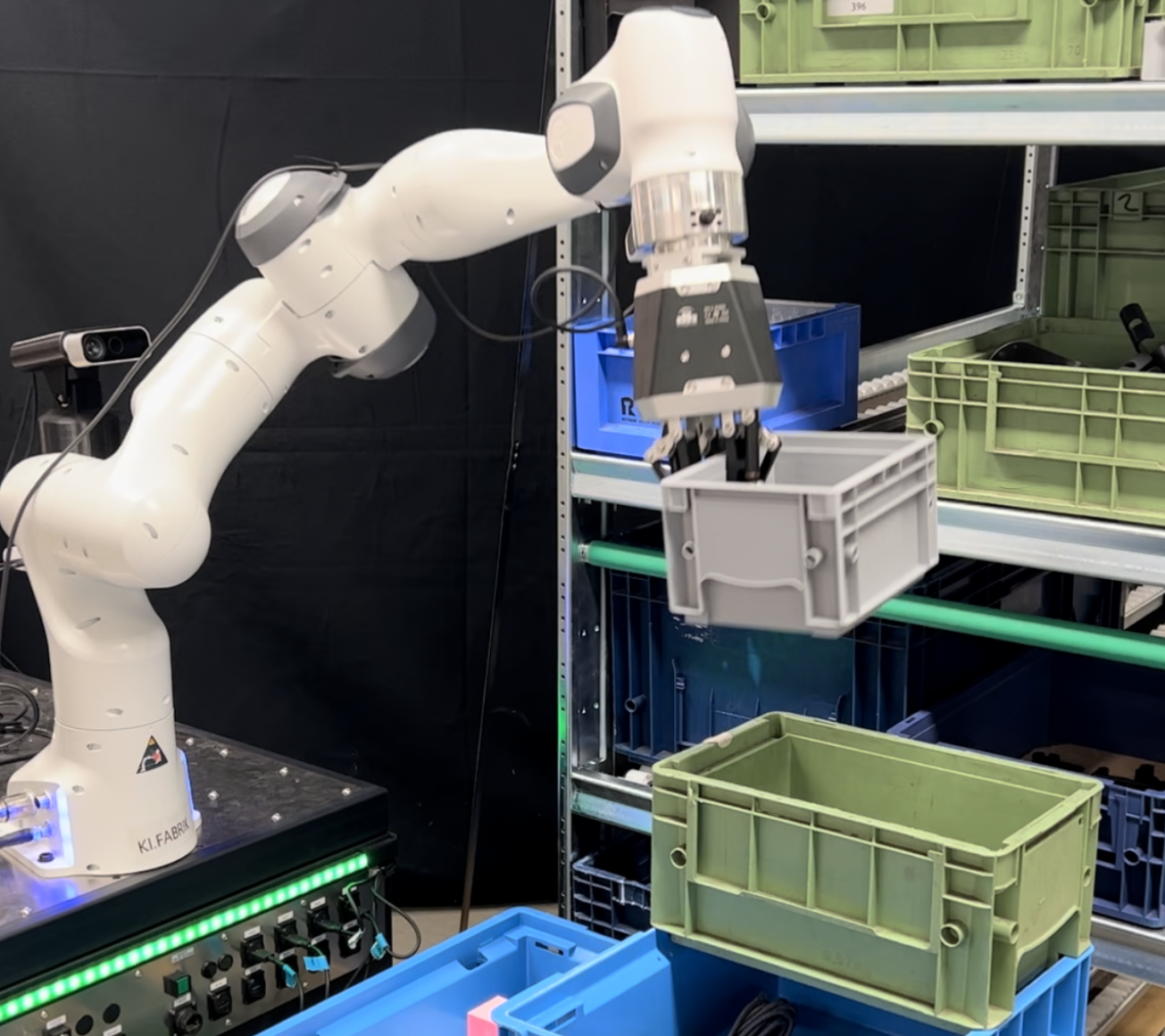}};
    \node[inner sep=0pt] (russell) at (5.5,0)
    {\includegraphics[width=0.244\textwidth]{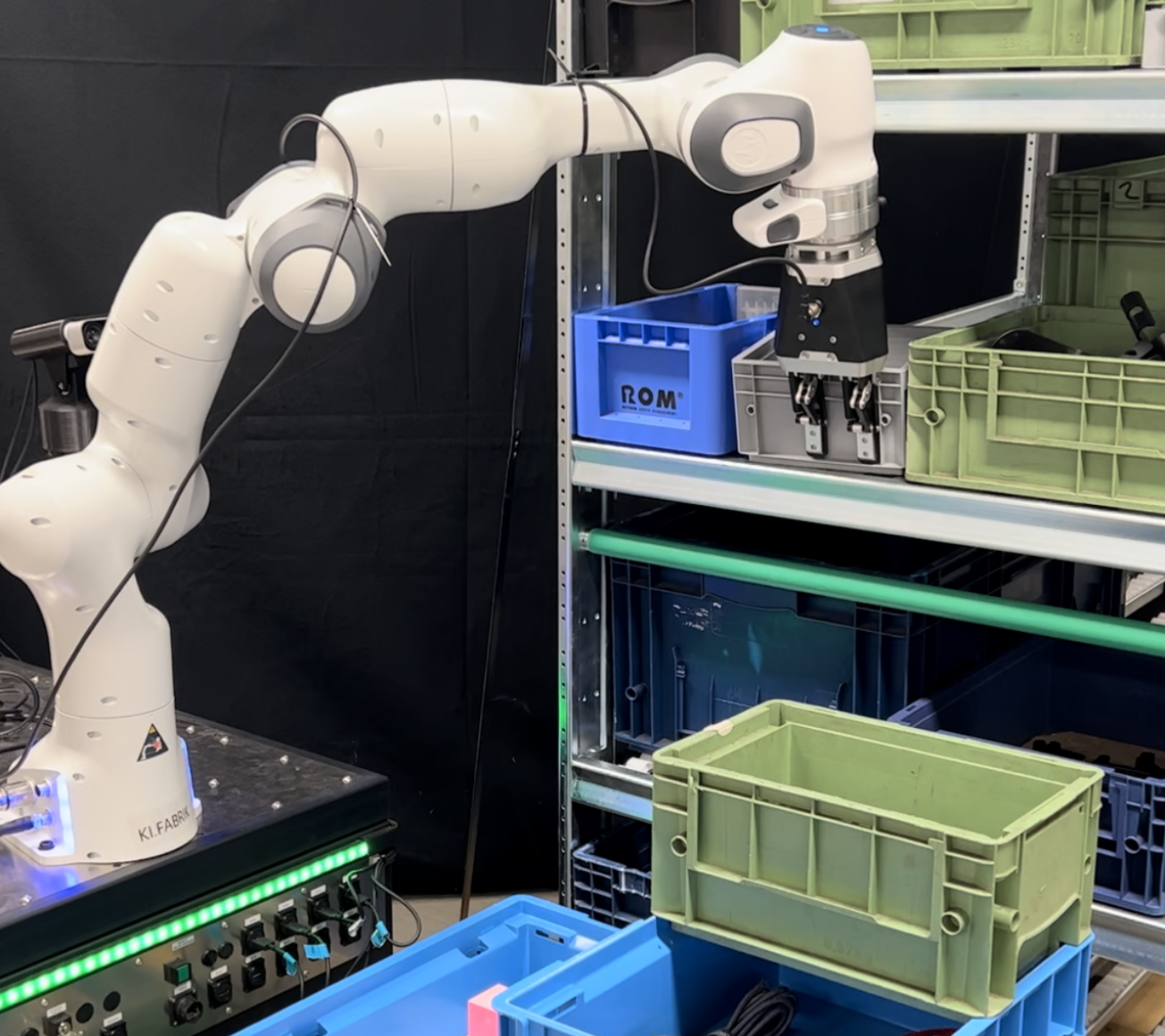}};
    
    \node at (-8.0,-2.4) {\small (a) Simulation scenery};
    \node at (-3.5,-2.4) {\small (b) Start configuration};
    \node at (1.0,-2.4) {\small (c) Transition configuration};
    \node at (5.5,-2.4) {\small (d) Goal configuration};
    
    \end{tikzpicture}
    \caption{Illustrates the simulation (a) and the real-world scenarios of DARKO robot for the industry shelf-arrangement task, (b) shows the start configuration of the arm to pick up the gray container from the pallet, (c) shows the transition configuration position of the task, (d) shows the goal configuration of the arm to place the container into the \textit{narrow corridor} (tolerances $\pm$5mm) between two other containers.}
    \label{fig: simulation}
    \vspace{-1.5em}
\end{figure*}
\subsection{Extension of Coulomb's Law to High Dimensions}\label{subsec: coulomb}
Consider a \(n\)-dimensional Euclidean space \( \mathbb{R}^n \). Let two vertex charges, \( q_1 \) and \( q_2 \), be located in this space, separated by a distance \( r \). The unit vector from \( q_1 \) to \( q_2 \) is denoted by \( \mathbf{\hat{r}} \). The Coulomb force in high-dimension is proposed to be:

\begin{equation}
\vec{F} := \mathbf{k}_e \frac{q_1 q_2}{r^{n-1}} \mathbf{\hat{r}},
\end{equation}
where \( \vec{F} \) represents the force vector exerted by the charges on each other. The proportionality constant \( \mathbf{k}_e \) is analogous to Coulomb's constant in three dimensions, determining the strength of the force. \( n \) represents the dimensionality of the \textit{$\mathcal{C}$-space} where these interactions occur, and \( r \) signifies the Euclidean distance between the charges.

\subsection{Elliptical $K$-Nearest Neighbors Search Method}\label{subsec: Eliptical}
The elliptical-KNN search method deviates from the conventional spherical search area and instead utilizes an elliptical or, more generally, an $n$-dimensional ellipsoidal search space~\cite{ohmori2021predictive}.
The uniqueness of the elliptical search lies in its adaptive nature (Fig.~\ref{fig:search_areas}), where the search area's shape dynamically aligns with the direction and magnitude of the Coulomb forces acting on the current state. 
%
\begin{Definition}
    Given two vertices in an n-dimensional space, $x$ and $y$, where $x = (x_1, x_2, \ldots, x_n)$ and $y = (y_1, y_2, \ldots, y_n)$, the elliptical distance metric is defined as:
        \begin{equation}
        d_{\text{elliptical}}(x, y) := \sqrt{\sum_{i=1}^{n}\left(\frac{x_i - y_i}{v_i}\right)^2},
        \end{equation}
    where $v_i$ represents the scaling factor for the $i$-th dimension, adjusting the importance and scale of different dimensions.
\end{Definition}

\begin{Theorem}
(Dynamic selection for elliptical-KNN):
Let $P$ be a set of vertices in an n-dimensional space and $q$ a query vertex. Assuming the distance metric is defined by the elliptical equation. The number of nearest neighbors, $k$, is dynamically determined by the function (\ref{eq: knn}), then a vertex $p \in P$ is an elliptical-KNN of $q$ if and only if for all $p' \in P$, $p$ satisfies the following condition for a number less than or equal to $k(q)$:
        \begin{equation}
        \sum_{i=1}^{n}\left(\frac{q_i - p_i}{v_i}\right)^2 \leq \sum_{i=1}^{n}\left(\frac{q_i - p'_i}{v_i}\right)^2,
        \end{equation}
\end{Theorem}
\begin{proof}
To prove that this selection optimizes the search process, we consider the continuous function \(f(q) = k(q)\) and its derivative with respect to \(q\),
    \begin{equation}
        k^{\prime}(q) = \frac{df(q)}{dq} = \eta e \left(1 + \frac{1}{n}\right) \frac{1}{q}.
    \end{equation}
This derivative indicates that as \(q\) increases, the rate of change of \(k(q)\) decreases, suggesting that \(k(q)\) grows logarithmically with \(q\).
The optimization comes from the fact that in denser regions of the space (higher \(q\)), a minor increase in \(k\) is needed to capture a representative set of nearest neighbors due to the logarithmic relationship. This ensures the search is adaptive and efficient, focusing computational resources where they are most needed.
By adjusting \(k\) dynamically, the search process can accommodate the varying density of vertices in different regions of the space, leading to an optimized balance between accuracy and computational cost.
Hence, the dynamic selection of \(k\) based on \(q\) is shown to optimize the exploration under the defined conditions.
\end{proof}
In high-dimensional path planning, the elliptical-KNN search (Alg.~\ref{alg:elliptcal_search},~\ref{alg:ellipse_neighbors}, and~\ref{alg:ellipseNearest}) becomes an ellipsoidal search, where the corresponding force component in that dimension weights each axis of the ellipsoid. This weighting inherently incorporates the relative information between the current state and other samples in the state space. 
The search becomes directional and informed.
Prioritizing the search effort on areas more likely to yield an improved path quality.

The directionality of the search is a direct consequence of the force direction, which not only enhances the algorithm's efficiency by guiding it towards obstacle-free regions but also encodes valuable information about the state space's topology and obstacle distribution. 
This enables FDIT* to navigate complex environments efficiently by leveraging force information for decision-making on search directions.

\section{Formal Analysis}~\label{sec: formalAnalysis}
In this paper, we refer to Definition 24 from~\cite{karaman2011sampling} to establish the concept of almost-sure asymptotic optimality. 

\subsection{Almost-Sure Asymptotically Optimal Path}
The FDIT* algorithm utilizes an increasingly dense RGG approximation similar to EIT*, in which the underlying graph almost certainly includes an asymptotically optimal path. The search in FDIT* is also asymptotically resolution-optimal. Considering EIT* is an almost-surely asymptotically optimal planner~\cite{gammell2022planner}. Therefore, FDIT*'s RGG approximation almost certainly contains an optimal path.
\subsection{Elliptical-KNN vs. Traditional KNN}
%
%
\subsubsection{Adaptive search area vs. fixed radius}Traditional KNN search uses a fixed-radius spherical area. In contrast, FDIT* employs an adaptive elliptical-KNN search, adjusting the ellipsoidal search space based on the net force direction and magnitude. This dynamic adaptation enables FDIT* to focus the search more effectively in the direction indicated by the current state's force vector.


\subsubsection{Directional information and relative weighting}The elliptical-KNN search incorporates directional information and relative weighting based on the force components in each dimension. In contrast, the traditional KNN search treats all regions equally, inadequate the capability to prioritize search directions based on the state space's dynamics or topology. 
The elliptical search inherently includes information about other vertices' relative positions to the current state, enhancing the search's relevance and efficiency.

\subsubsection{Search efficiency in high-dimensional spaces}
The elliptical-KNN search incorporates directional information and relative weighting based on force components in each dimension. Unlike traditional KNN search, which treats all regions equally, it can prioritize search directions based on the state space's dynamics and topology. This approach inherently includes information about other vertices' relative positions, enhancing the search's relevance and efficiency.


\subsubsection{Improved path optimization}
By configuring the NN search area with invalid vertices and the resultant force direction, the algorithm is probable to find valuable sampled states during the NN searching process, which form paths that are both feasible and optimal, particularly in environments with numerous obstacles and high-dimensional \textit{$\mathcal{C}$-space}.
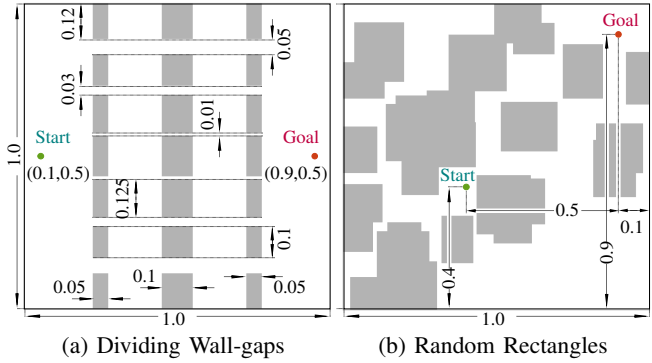
\begin{figure}[t!]
    \centering
    \begin{tikzpicture}
    \node[inner sep=0pt] (russell) at (-4.0,0.0)
    {\includegraphics[width=0.24\textwidth]{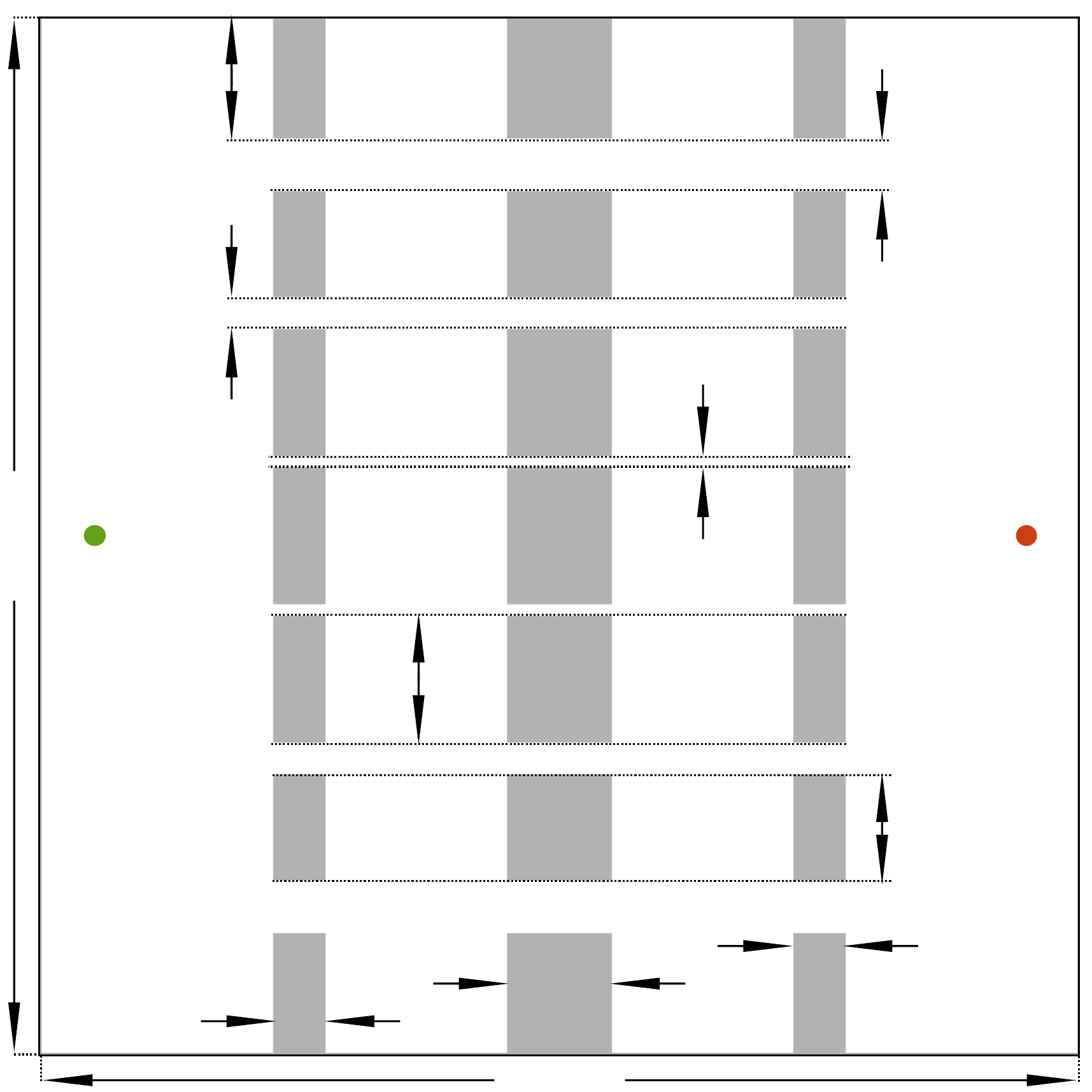}};
    \node[inner sep=0pt] (russell) at (0.25,0.0)
    {\includegraphics[width=0.24\textwidth]{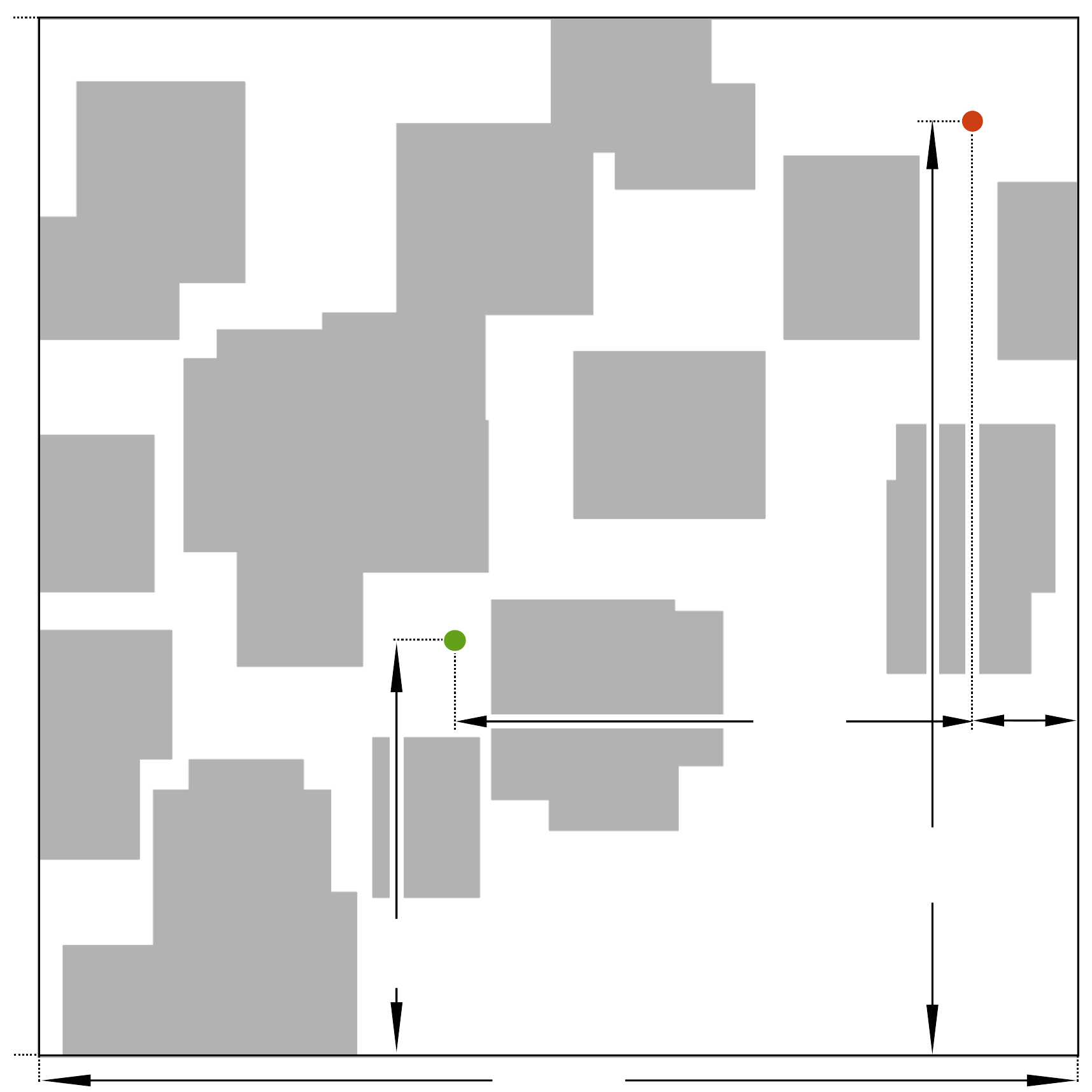}};
    \scriptsize
    
    \node [rotate=90] at (-5.4,0.93) {0.03};    
    \node [rotate=90] at (-4.68,-0.52) {0.125};    
    \node [rotate=90] at (-5.4,1.8) {0.12};
    \node [rotate=90] at (-2.53,-1.05) {0.1};
    \node [rotate=90] at (-3.55,0.6) {0.01};
    \node [rotate=90] at (-2.53,1.53) {0.05};
    \node at (-5.38,-1.7) {0.05};
    \node at (-4.38,-1.55) {0.1};
    \node at (-2.45,-1.7) {0.05};
    \node [rotate=90] at (-6.12,0.05) {1.0};
    \node at (-4.0,-2.13) {1.0};
    \node at (-5.53,-0.2) {(0.1,0.5)};
    \node at (-2.37,-0.2) {(0.9,0.5)};
    \node at (-5.6,0.3) {\color{teal} Start};
    \node at (-2.3,0.3) {\color{purple} Goal};
    \node at (1.24,-0.68) {0.5};
    \node at (0.25,-2.13) {1.0};
    \node [rotate=90] at (-0.34,-1.58) {0.4};
    \node [rotate=90] at (1.76,-1.24) {0.9};
    \node  at (2.1,-0.9) {0.1};

    \node at (-0.22,-0.21) {\color{teal} Start};
    \node at (1.95,1.85) {\color{purple} Goal};

    \node at (-4.0,-2.51) {\small (a) Dividing Wall-gaps};
    \node at (0.25,-2.51) {\small (b) Random Rectangles};
    \end{tikzpicture}
    \caption{Simulated planning problems were visualized using a 2D representation. The possible states, represented by $X \subset \mathbb{R}^n$, were restricted to a hypercube with a side length of 1.0 for both problems. Ten variations of the dividing wall-gaps and random rectangles experiment were tested, and the results are shown in Fig.~\ref{fig: result}.}
    \label{fig: testEnv}
    \vspace{-1.7em} 
\end{figure}
\subsection{Synthetic Electromagnetic Guidance Method}\label{subsec:guidance}
A virtual electromagnetic field is constructed with the target vertex as the center. The current vertex's real-time position coordinates are denoted as \( x \), with the target vertex exerting an electromagnetic attraction or repulsion on the current vertex, represented as \( F_{\text{attractive}}(x) \), and the electromagnetic potential energy denoted as \( \mathcal{E}_{\text{attractive}}(x) \). Surrounding vertices in obstacles exert an electromagnetic repulsive force on the current vertex, represented as \( F_{\text{repulsive}}(x) \), with the electromagnetic repulsive potential energy denoted as \( \mathcal{E}_{\text{repulsive}}(x) \).
The electromagnetic potential energy at the current state's position \( x \), \( \mathcal{E}_{\text{attractive}}(x) \), is expressed as follows:
\begin{equation}
    \mathcal{E}_{\text{attractive}}(x) = \mathbf{k}_a \cdot \frac{1}{\left|x - x_{\text{positive}}\right|^2},
\end{equation}
where \( \mathbf{k}_a \) is the Coulomb’s charge coefficient related to the positive force, and \( x_{\text{positive}} \) denotes the location of the positive (valid) vertex.
The electromagnetic force (Alg.~\ref{computationforcedirection}, line~\ref{line:forceAttra}) exerted on the current vertex at position \( x \), \( F_{\text{attractive}}(x) \) , is:
\begin{equation}
    F_{\text{attractive}}(x) = \mathbf{k}_a \cdot \frac{1}{\left|x - x_{\text{positive}}\right|^3} \cdot (x - x_{\text{positive}}),
\end{equation}

The electromagnetic repulsive potential energy at the current vertex's position \( x \) is represented as \( \mathcal{E}_{\text{repulsive}}(x) \), where \( \mathbf{k}_r \) is the charge coefficient related to the negative vertex, \( x_{\text{negative}} \) is the negative (invalid) vertex's position, and \( \rho_0 \) is the range within which the negative vertex can exert electromagnetic repulsion on the current vertex, is:
\begin{equation}
    \mathcal{E}_{\text{repulsive}}(x) = \left\{\begin{array}{cl}
        -\mathbf{k}_r \cdot \frac{1}{\left|x - x_{\text{negative}}\right|^2} & , \left|x - x_{\text{negative}}\right| \leq \rho_0 \\
        0 & , \left|x - x_{\text{negative}}\right| > \rho_0
    \end{array}\right.,
\end{equation}

The electromagnetic repulsive force on the current vertex at position \( x \), \( F_{\text{repulsive}}(x) \), is:
\begin{equation}
    F_{\text{repulsive}}(x) = \left\{\begin{array}{cl}
        -\mathbf{k}_r \cdot \frac{(x - x_{\text{negative}})}{\left|x - x_{\text{negative}}\right|^3} & , \left|x - x_{\text{negative}}\right| \leq \rho_0 \\
        0 & , \left|x - x_{\text{negative}}\right| > \rho_0
    \end{array}\right..
\end{equation}

Through these adjustments, we have constructed a virtual potential field model based on Coulomb's law suitable for electromagnetic field environments, for path planning of the current vertex. Considering the interactions between charges, including attraction and repulsion, this model enables effective navigation of the current vertex in an environment containing positive and negative vertices.
The following forms for electromagnetic potential energy and force expressions in high-dimensional state spaces are proposed:

\begin{equation}
\vec{\mathcal{E}}_{\text{attractive},i}(\mathbf{x}) = \mathbf{k}_e \sum_{i=1}^{m_{\text{valid}}} \frac{q_1 q_{2, i}}{\|\mathbf{x} - \mathbf{x}_{\text{positive}, i}\|^{n-1}},
\end{equation}
\begin{equation}
\vec{\mathcal{E}}_{\text{repulsive},j}(\mathbf{x}) = -\mathbf{k}_e \sum_{j=1}^{m_{\text{invalid}}} \frac{q_1 q_{2, j}}{\|\mathbf{x} - \mathbf{x}_{\text{negative}, j}\|^{n-1}},
\end{equation}
where \( m_{\text{valid}} \) and \( m_{\text{invalid}} \) are the number of valid and invalid sample vertices in \textit{$\mathcal{C}$-space}, respectively.

\begin{equation}
\vec{F}_{\text{attractive},i}(\mathbf{x}) = \mathbf{k}_e \sum_{i=1}^{m_{\text{valid}}} \frac{q_1 q_{2, i}}{\|\mathbf{x} - \mathbf{x}_{\text{positive}, i}\|^{n}} (\mathbf{x}_{\text{positive}, i} - \mathbf{x}),
\end{equation}
\begin{equation}
\vec{F}_{\text{repulsive},j}(\mathbf{x}) = -\mathbf{k}_e \sum_{j=1}^{m_{\text{invalid}}} \frac{q_1 q_{2, j}}{\|\mathbf{x} - \mathbf{x}_{\text{negative}, j}\|^{n}} (\mathbf{x}_{\text{negative}, j} - \mathbf{x}).
\end{equation}

In these formulas:
\begin{itemize}
    \item \( \mathbf{k}_e \) is a proportionality constant to Coulomb's constant.
    \item \( q_1 \) is the charge of the current state, \( q_{2, i} \) and \( q_{2, j} \) are the charges of the \( i \)-th valid sample vertex and the \( j \)-th invalid sample vertex, respectively.
    \item \( \mathbf{x} \) is the position vector of the current state, \( \mathbf{x}_{\text{positive}, i} \) and \( \mathbf{x}_{\text{negative}, j} \) are the position vectors of the \( i \)-th valid vertex and the \( j \)-th invalid vertex, respectively.
    \item \( n \) is the dimensionality of the \textit{$\mathcal{C}$-space}.
\end{itemize}

In contemplating the combined influence of numerous sample vertices on the present condition, these expressions prove versatile  within high-dimensional state spaces.
%
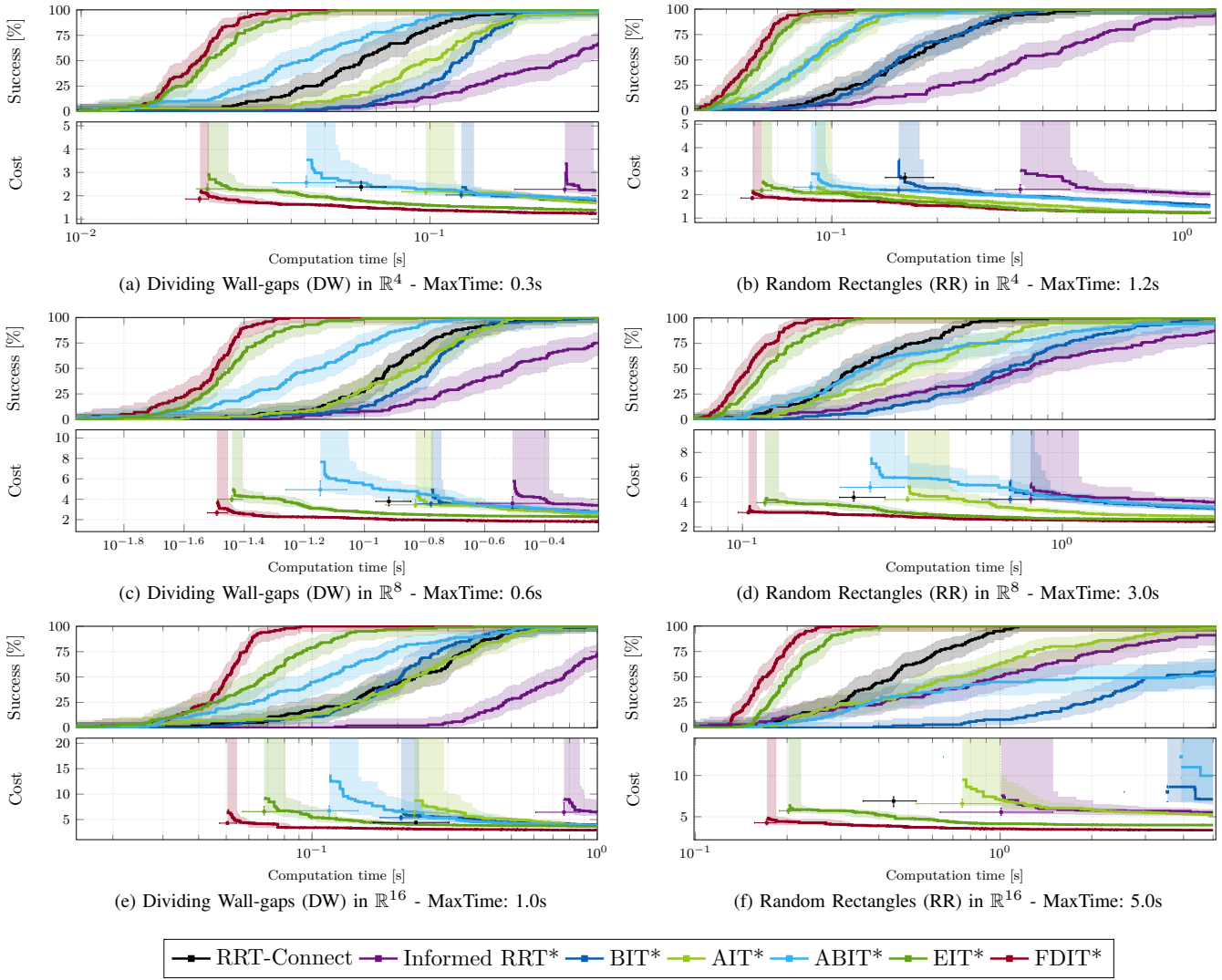
\begin{figure*}[t!]
    \centering
    \begin{tikzpicture}
    \node[inner sep=0pt] (russell) at (4.1,8)
    {\includegraphics[width=0.49\textwidth]{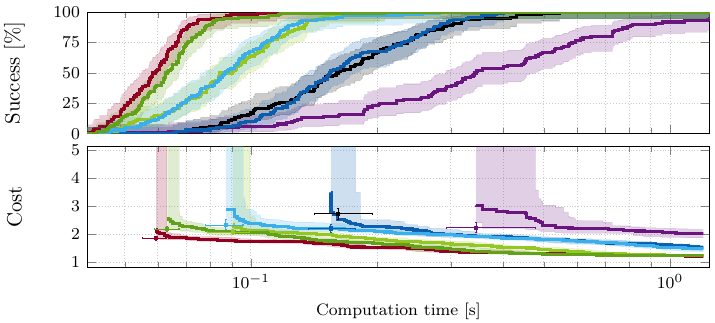}};
    \node[inner sep=0pt] (russell) at (4.1,3.5)
    {\includegraphics[width=0.49\textwidth]{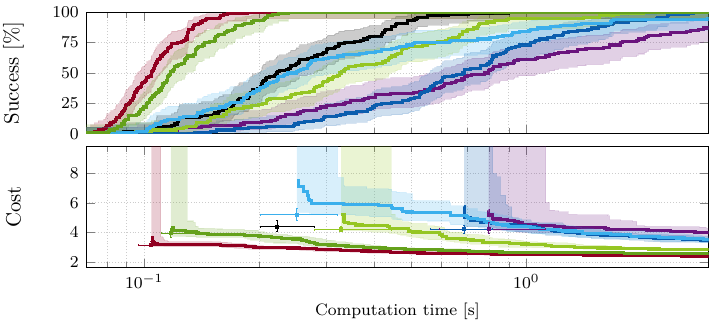}};
    \node[inner sep=0pt] (russell) at (4.1,-1)
    {\includegraphics[width=0.49\textwidth]{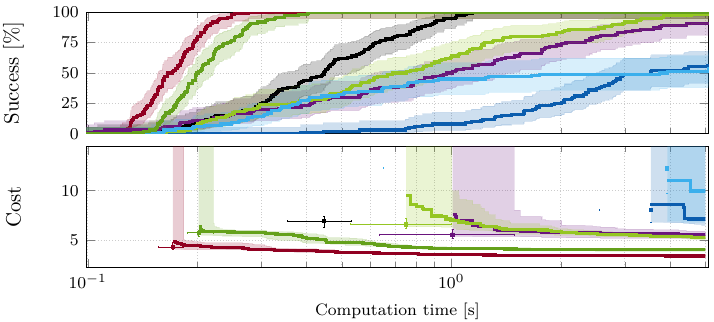}};

    \node[inner sep=0pt] (russell) at (-4.9,8)
    {\includegraphics[width=0.49\textwidth]{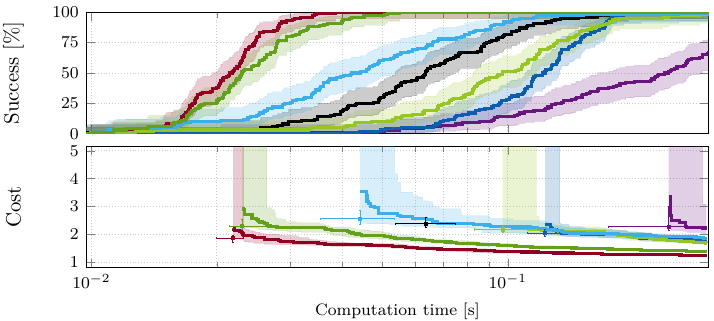}};
    \node[inner sep=0pt] (russell) at (-4.9,3.5)
    {\includegraphics[width=0.49\textwidth]{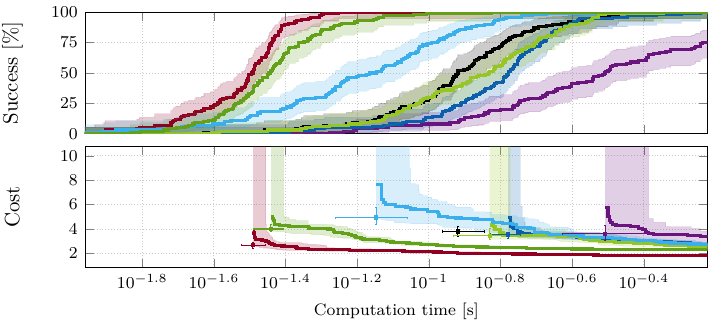}};  
    \node[inner sep=0pt] (russell) at (-4.85,-1){\includegraphics[width=0.496\textwidth]{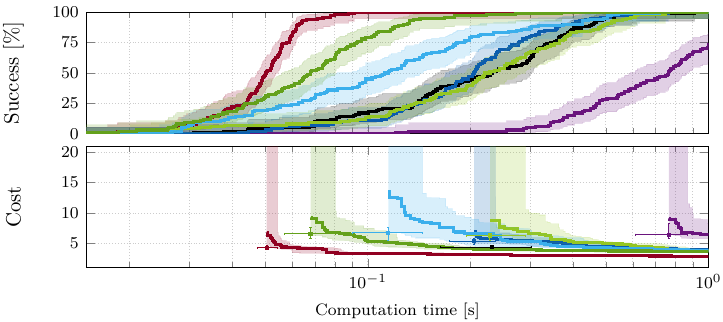}};

    \node[inner sep=0pt] (russell) at (0.0,-4.05){\includegraphics[width=0.8\textwidth]{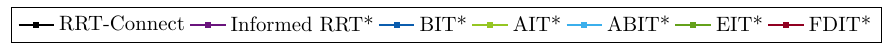}};

    \node at (-4.5,5.8) {\footnotesize (a) Dividing Wall-gaps (DW) in $\mathbb{R}^4$ - MaxTime: 0.3s};
    \node at (-4.5,1.3) {\footnotesize(c) Dividing Wall-gaps (DW) in $\mathbb{R}^8$ - MaxTime: 0.6s};
    \node at (-4.5,-3.2) {\footnotesize(e) Dividing Wall-gaps (DW) in $\mathbb{R}^{16}$ - MaxTime: 1.0s};

    \node at (4.5,5.8) {\footnotesize (b) Random Rectangles (RR) in $\mathbb{R}^4$ - MaxTime: 1.2s};
    \node at (4.5,1.3) {\footnotesize(d) Random Rectangles (RR) in $\mathbb{R}^8$ - MaxTime: 3.0s};
    \node at (4.5,-3.2) {\footnotesize(f) Random Rectangles (RR) in $\mathbb{R}^{16}$ - MaxTime: 5.0s};

    \end{tikzpicture}
    \vspace{-1.0em} 
    \caption{Detailed experimental results from Section~\ref{subsec:experi} are presented above. Fig. (a), (c), and (e) depict test benchmark dividing wall gaps outcomes in $\mathbb{R}^4$, $\mathbb{R}^8$ and $\mathbb{R}^{16}$, respectively. Panel (b) showcases ten random rectangle experiments in $\mathbb{R}^4$, while panels (d) and (f) demonstrate in $\mathbb{R}^8$ and $\mathbb{R}^{16}$. In the cost plots, boxes represent solution cost and time, with lines showing cost progression for planners (unsuccessful runs have infinite costs). Error bars provide nonparametric 99\% confidence intervals for solution cost and time.}
    \label{fig: result}
    \vspace{-1.8em}
\end{figure*}

\section{Experimental Results}\label{sec:experi}

FDIT* was evaluated against several existing algorithms, including different versions of RRT-Connect, Informed RRT*, BIT*, AIT*, ABIT*, and EIT* from the Open Motion Planning Library (OMPL). Tests were conducted in simulated environments ranging from $\mathbb{R}^4$ to $\mathbb{R}^{16}$ and real-world manipulation scenarios using an Intel i7 3.90 GHz processor with 32GB of LPDDR3 3200 MHz memory. The main goal was to minimize the median initial path length ($c^\textit{med}_\textit{init}$) over 100 runs. For all planners, the RGG constant $\eta$ was set to 1.1, and the rewire factor to 1.001. RRT-based algorithms used a 5\% goal bias, with maximum edge lengths adjusted for space dimensionality. FDIT*, BIT*, AIT*, ABIT*, and EIT* sampled 200 states per batch, utilizing Euclidean distance and effort heuristics with graph pruning. Unlike the others, FDIT* employed elliptical-KNN for NN search and rewiring.

FDIT*'s elliptical-KNN search method dynamically adjusted the search regions. The planner's adaptive mechanism optimizes the NN search areas and rewires at any time.

\subsection{Experimental Tasks}\label{subsec:experi}
The planners were subjected to testing across three distinct problem domains: $\mathbb{R}^4$, $\mathbb{R}^8$, and $\mathbb{R}^{16}$. In the first scenario, a constrained environment resembling a dividing wall with several narrow gaps was simulated, allowing valid paths in multiple general directions for non-intersecting solutions (Fig. \ref{fig: testEnv}a). Each planner underwent 100 runs, and the computation time for each anytime asymptotically optimal planner is demonstrated in the labels, with varying random seeds. The overall success rates and median path lengths for all planners are depicted in Fig.~\ref{fig: result}a, \ref{fig: result}c, and \ref{fig: result}e.
\begin{table}[t]
\caption{Benchmarks evaluation comparison (Fig.~\ref{fig: result})}
\centering
\resizebox{0.485\textwidth}{!}{
\begin{tabular}{p{1.4cm}||c|c|c|c|c|c|c}
 \hline
 & \multicolumn{3}{c}{${\text{EIT*}}$} & \multicolumn{3}{|c|}{$\textcolor{purple}{\text{FDIT*}}$} &\multirow{2}*{\large$c^\textit{med}_\textit{init}\color{purple}\Uparrow$ (\%)}\\
    &$t^\textit{med}_\textit{init}$ &$c^\textit{med}_\textit{init}$ &$c^\textit{med}_\textit{final}$ &$t^\textit{med}_\textit{init}$ &$c^\textit{med}_\textit{init}$ &$c^\textit{med}_\textit{final}$ \\
 \hline

    $\text{RR}-\mathbb{R}^4$   &0.0632   &2.1817   &1.2286 &0.0593 &1.8532 &1.2173 &15.06  \\
    $\text{RR}-\mathbb{R}^8$   &0.1175   &3.9503   &2.6092 &0.1043 &\textcolor{purple}{3.1590} &2.4199  &{20.03}  \\
    $\text{RR}-\mathbb{R}^{16}$   
    &{0.2019} &{5.7269} &{3.9973} &{0.1715} &\textcolor{purple}{4.2738} &\textcolor{purple}{3.3797}  &\textcolor{purple}{25.37}\\
    
    $\text{DW}-\mathbb{R}^4$    &0.0230 &2.2950 &1.3820 &0.0218 &1.8600 &1.2332  &18.95\\
    $\text{DW}-\mathbb{R}^8$  &0.0362  &4.0013  &2.2889 &0.0323 &\textcolor{purple}{2.6774} &\textcolor{purple}{1.8085} &\textcolor{purple}{33.34}  \\
    $\text{DW}-\mathbb{R}^{16}$   &0.0679   &6.5623   &3.6617 &\textcolor{purple}{0.0505} &\textcolor{purple}{4.2887} &\textcolor{purple}{2.9037}  &\textcolor{purple}{\textbf{34.65}}  \\
 \hline
\end{tabular}} \label{tab:benchmark}
\vspace{-1.3em} 
\end{table}

In the second test environment, random widths were assigned to \textit{axis-aligned hyperrectangles}, which were generated within the \textit{$\mathcal{C}$-space} arbitrarily (Fig. \ref{fig: testEnv}b). Unique random rectangle problems were created for each dimension of the \textit{$\mathcal{C}$-space}, and each planner was run 100 times for every instance. Fig.~\ref{fig: result}b, \ref{fig: result}d, and \ref{fig: result}f depict the success rates and median path costs within the computation time.

As observed in Table \ref{tab:benchmark}, there's a median cost improvement across varied benchmark scenarios, correlating with dimensionality. In the case of the DW-$\mathbb{R}^{16}$ scenario, the initial median solution cost exhibits a reduction of up to 34.65\%.
\subsection{Path Planning for DARKO}
FDIT* demonstrated its effective elliptical-KNN techniques during an industry container and multi-layer shelf arrangement tasks (Fig.~\ref{fig: darko_setup} and Fig.~\ref{fig: simulation}). DARKO (8-DoF) is a mobile manipulation robotic platform designed to tackle challenges within intralogistics supermarket cells. 
%
Inserting industry containers into narrow spaces is difficult due to component standardization. 
The task is to place the industry standard container between two larger containers on the shelf. Due to the tolerance scope between containers being $\leq$5mm, planning a collision-free feasible path is challenging.
%
%
This container-shelf arrangement task aims to optimize path length and success rate. All planners had 1.0 seconds to address this problem. Over 15 trials, FDIT* was 80\% successful with a median solution cost of 10.9225. EIT* was 73.3\% successful and had a median solution cost of 14.5726. AIT* was 60\% successful with a median solution cost of 16.7038, and ABIT* was 66.67\% successful with a median solution cost of 16.4719. FDIT* consistently optimizes the elliptical-KNN by integrating the information of \textit{invalid vertices} in these \textit{$\mathcal{C}$-space} obstacles. This allows FDIT* to become more problem-specific and achieve higher success rates in completing tasks. 
The detailed behavior of DARKO can be viewed in \href{https://youtu.be/s7CHNLbVo6k}{\textcolor{blue}{https://youtu.be/s7CHNLbVo6k}}.
\section{Discussion \string& Conclusion}

In this paper, we introduce FDIT*, an advanced sampling-based planner using elliptical-KNN search, which improves on traditional KNN by including invalid vertices and calculating force directions with Coulomb's law. This method enhances path planning algorithms.
Determining vertex charges \( q_{2,i} \) and \( q_{2,j} \) is crucial, as it affects their influence on the current state \( q_1 \). Currently, all vertices have the same charge, but practical applications may require different charges based on potential values. Future research will focus on this quantification.
In tests with the DARKO robot, FDIT* demonstrated its adaptability by achieving short path lengths and quickly generating initial solutions, especially in narrow environments. This showcases its practical effectiveness.

In conclusion, FDIT* uses invalid vertices and physical force dynamics to focus search regions and reduce costs. In the future, we could further test FDIT* in bio-inspired robots~\cite{Bing2023SR}, consider incorporating human acceptability and comfort into planning~\cite{cai2022human}, and investigate subset optimization methods~\cite{zhang2025TASE} for robot motion planning.

\bibliographystyle{IEEEtran}
\bibliography{references}
\balance

\end{document}